\documentclass[conference]{IEEEtran}
\IEEEoverridecommandlockouts
\usepackage{amsmath,amssymb,amsfonts}
\usepackage{amsthm}
\usepackage{bm}
\usepackage{graphicx}
\usepackage{textcomp}
\usepackage{mathptmx}
\usepackage{amsmath}
\usepackage{tabularx}
\usepackage{subfigure}
\usepackage{xcolor}
\usepackage{cite}
\usepackage{multirow}
    
\graphicspath{{./images/}}
    
\newtheorem{assumption}{Assumption}
\newtheorem{theorem}{Theorem}
\newtheorem{corollary}{Corollary}
\newtheorem{proposition}{Proposition}
\newtheorem{lemma}{Lemma}

\newtheorem{remark}{Remark}
\usepackage{url} 
\usepackage{algorithm,algpseudocode}

\def\BibTeX{{\rm B\kern-.05em{\sc i\kern-.025em b}\kern-.08em
    T\kern-.1667em\lower.7ex\hbox{E}\kern-.125emX}}

\begin{document}

\title{\huge{Federated Learning for Discrete Optimal Transport with Large Population under Incomplete Information}
}

\author{Navpreet Kaur, Juntao Chen, and Yingdong Lu
\thanks{N. Kaur and J. Chen are with the Department of Computer and Information Sciences, Fordham University. E-mail: \{nkaur15,jchen504\}@fordham.edu}%
\thanks{Y. Lu is with IBM Research. E-mail: yingdong@us.ibm.com}%
}

\maketitle

\begin{abstract}
Optimal transport is a powerful framework for the efficient allocation of resources between sources and targets. However, traditional models often struggle to scale effectively in the presence of large and heterogeneous populations. In this work, we introduce a discrete optimal transport framework designed to handle large-scale, heterogeneous target populations, characterized by type distributions. We address two scenarios: one where the type distribution of targets is known, and one where it is unknown. For the known distribution, we propose a fully distributed algorithm to achieve optimal resource allocation. In the case of unknown distribution, we develop a federated learning-based approach that enables efficient computation of the optimal transport scheme while preserving privacy. Case studies are provided to evaluate the performance of our learning algorithm.
\end{abstract}


\section{Introduction}
\label{sec:intro}

Optimal transport is a fundamental framework used to design efficient resource distribution schemes between two parties, typically referred to as sources and targets~\cite{peyre2019computational}. Its application spans various domains, including supply chains (e.g., distributing raw materials from factories to manufacturers), task allocation (e.g., assigning tasks to employees), and more recently, domain adaptation in machine learning \cite{courty2016optimal}.

However, the traditional optimal transport paradigm assumes that complete information about all participants is available and that the transport network can be explicitly characterized~\cite{hughes2021fair}. 
This assumption is limiting in modern contexts, such as supply chains and machine learning, where the scale of the network can be vast and dynamic, involving many target nodes requesting resources. In such scenarios, it is often infeasible for resource providers to acquire comprehensive knowledge about the preferences of all target nodes. Therefore, the standard optimal transport framework must be extended to accommodate large-scale, heterogeneous populations of targets with incomplete information.

In this paper, we present a new optimal transport framework designed to address the challenges posed by large populations with varying preferences. The heterogeneity of target nodes is represented by a type distribution function, categorizing target nodes based on their preferences for resources. We focus on two key settings: one where the target type distribution is known and one where it is unknown. In the first case, we propose a fully distributed algorithm that optimally allocates resources among the target nodes. In the second case, where the target type distribution is unavailable to the source nodes (transport planners), we introduce a federated learning approach. This approach enables the source nodes to collaboratively and efficiently update transport schemes as new information about the target nodes is gradually collected.

Our federated learning algorithm is particularly advantageous in scenarios where privacy is a concern, as it allows each target node to calculate local solutions without sharing private data directly with the central planner \cite{9141214}. Instead, the local solutions are aggregated to form a global transport plan. This method has practical applications in privacy-preserving systems~\cite{NAGY2023110475} and mobile computing environments~\cite{9195793}.

The paper is organized as follows. Section \ref{sec:framework} establishes the large-scale discrete optimal transport framework for resource allocation. Section \ref{sec:OT_known} develops a distributed algorithm to compute the optimal transport plan when the target's type distribution is known, and Section \ref{sec:OT_unknown} proposes a federated learning algorithm to compute the solution when such information is unknown. Section \ref{sec:cases} presents case studies to showcase the developed mechanism. Section \ref{sec:conclusion} concludes the paper.

\subsection{Related Works}
\subsubsection{Federated Learning with Heterogeneous Data}
Extensive research has explored the application of federated learning algorithms to heterogeneous data. One of the primary challenges in federated learning is managing the inherent heterogeneity of data, which arises due to differences in statistical distributions, model architectures, and data representations across local devices. This variability complicates the process of aggregating local solutions at the central planner, as the non-uniform nature of the data makes it difficult to apply a standardized method across all local updates. Prior studies \cite{heterogeneous_federated_learning, li2020federated, a_survey_on_hetrogeneous} have identified these challenges, categorizing them into different forms of heterogeneity, including statistical heterogeneity, model heterogeneity, and data space heterogeneity. Various strategies have been proposed to address these issues. For example, \cite{li2020federated} and \cite{Ghosh2019RobustFL} offer solutions aimed at mitigating these complexities with \cite{Ghosh2019RobustFL} specifically introducing a clustering-based approach to enhance the robustness of the algorithm. In our work, we tackle these challenges by categorizing target nodes into distinct types based on probability distribution functions (PDFs). This categorization enables the federated learning algorithm to be efficiently applied within the optimal transport framework, even in scenarios involving large populations with incomplete information.

\subsubsection{Optimal Transport and Federated Learning}
Several studies have explored the integration of federated learning with the optimal transport framework to address the inherent challenges in federated learning. For instance, \cite{global_and_local}  introduces a method called Federated Prompts Cooperation via Optimal Transport, which employs prompt learning to tackle issues such as heterogeneity. Their approach involves learning in both local and global settings to capture target preferences while also establishing consensus among target nodes. Another significant challenge in federated learning is the non-identical distribution of data across participants. In this regard, \cite{9801676} proposes using optimal transport as a core learning algorithm to address these discrepancies, enhancing the efficiency and robustness of federated learning in such environments.

\section{Large-Scale Discrete Optimal Transport}\label{sec:framework}

In this section, we present the framework for discrete optimal transport for resource allocation over a large-scale network.

\subsection{Framework}
\subsubsection{{\bf Nodes: sources and targets}}
We consider a transport network with $N$ target nodes (resource receivers) and $M$ source nodes (resource providers). Since the number of target nodes $N$ is large, we categorize those nodes into different types. This type parameter is denoted by $x\in\mathcal{X}$, where the size of $\mathcal{X}$, $|\mathcal{X}|$, is finite. The number of type-$x$ target nodes is denoted as $n_x$ for $x\in\mathcal{X}$.
Type distribution of target nodes is denoted by $P_t$ with $P_t(x)$ representing the proportion of target nodes of type $x$, i.e., 
\begin{align*}
P_t(x) = \frac{n_x}{N}, \quad\forall x \in\mathcal{X}.
\end{align*}
We assume that $P_t(x)> 0$ for $\forall x\in\mathcal{X}$ (otherwise we can just remove this type) and $\sum_{x\in\mathcal{X}} P_t(x)= 1$. 
Further denote by $\mathcal{Y}:=\{1,2,...,M\}$ the set of source nodes that distribute resources to the targets in the network. We assume that $N\gg M$ due to the feature of large-population of target nodes of the network.

\subsubsection{{\bf Network Topology}}

For each source node $y\in\mathcal{Y}$, denote $\mathcal{X}_y$ as set of target nodes connected to $y$. 
We assume, for each target node type $x\in\mathcal{X}$, each node of type-$x$ connects to the same set of source nodes and it is denoted as $\mathcal{Y}_x$. We further assume that $\mathcal{X}_y$, $\forall y$ and $\mathcal{Y}_x$, $\forall x$ are nonempty. Otherwise, the corresponding nodes are isolated in the network and do not participate in the resource matching. 
For convenience, we denote by $\mathcal{E}$ the set including all feasible transport paths in the network, i.e., $\mathcal{E}:=\{\{x,y\}|x\in\mathcal{X}_y,y\in\mathcal{Y}_x\}$. Here, $\mathcal{E}$ also refers to the set of all edges in the established bipartite graph for resource transportation.

\subsubsection{{\bf Problem Formulation}}

Denote by $\pi_{xy}\in\mathbb{R}_+$ the amount of resources transported from the source node $y\in\mathcal{Y}$ to the destination of type $x\in\mathcal{X}$. Let $\Pi:=\{\pi_{xy}\}_{x\in\mathcal{X}_y,y\in\mathcal{Y}}$ be the designed transport plan. To this end, the centralized optimal transport problem for utility maximization can be formulated as follows:
\begin{equation}\label{OT1:eqn}
\begin{aligned}
    \max_{\Pi}\ \sum_{x\in\mathcal{X}} \sum_{y\in\mathcal{Y}_x}& t_{xy}(\pi_{xy})P_t(x)N + \sum_{y\in\mathcal{Y}} \sum_{x\in\mathcal{X}_y} s_{xy}(\pi_{xy})P_t(x)N\\
    \mathrm{s.t.}\quad &\underline{p}_{x}\leq \sum_{y\in\mathcal{Y}_x} \pi_{xy}\leq \bar{p}_{x},\ \forall x\in\mathcal{X},\\
    &\underline{q}_{y}\leq \sum_{x\in\mathcal{X}_y} \pi_{xy}P_t(x)N\leq \bar{q}_{y},\ \forall y\in\mathcal{Y},\\
    &\pi_{xy}\geq 0,\ \forall \{x,y\} \in\mathcal{E},
\end{aligned}
\end{equation}
where $t_{xy}:\mathbb{R}_+\rightarrow\mathbb{R}$ and $s_{xy}:\mathbb{R}_+\rightarrow\mathbb{R}$ are utility functions for target node $x$ and source node $y$, respectively. Furthermore, $\bar{p}_x\geq \underline{p}_{x}\geq 0$, $\forall x\in\mathcal{X}$ and $\bar{q}_y\geq \underline{q}_{y}\geq 0$, $\forall y\in\mathcal{Y}$. The constraints $\underline{p}_{x}\leq \sum_{y\in\mathcal{Y}_x} \pi_{xy}\leq \bar{p}_{x}$ and $\underline{q}_{y}\leq \sum_{x\in\mathcal{X}_y} \pi_{xy}P_t(x)N \leq \bar{q}_{y}$ capture the limitations on the amount of requested and transferred resources at the type $x$ target node and source node $y$, respectively.

We have the following assumption on the utility functions.

\begin{assumption}\label{assumption:assum1}
The utility functions $t_{xy}$ and $s_{xy}$ are increasing concave functions of $\pi_{xy}$, $\forall x\in\mathcal{X},\forall y\in\mathcal{Y}$.
\end{assumption}

\section{Distributed Optimal Transport with Known Target's Type Distribution}\label{sec:OT_known}
Here, we establish the distributed algorithm for the formulation in \eqref{OT1:eqn}. Our first step is to rewrite the optimization problem in the ADMM form by introducing ancillary variables $\pi_{xy}^t$ and $\pi_{xy}^s$. The superscripts $t$ and $s$ indicate that the corresponding parameters belong to the target node or the source node, respectively. We then set $\pi_{xy} = \pi_{xy}^t$ and $\pi_{xy} = \pi_{xy}^s$, indicating the solutions proposed by the targets and sources are consistent with the ones proposed by the central planner. This reformulation facilitates the design of a distributed algorithm which allows us to iterate through the process in obtaining the optimal transport scheme. To this end, the reformulated optimal transport problem is presented as follows:
\begin{equation}\label{OT2:eqn}
\begin{aligned}
\min_{\Pi_t \in \mathcal{F}_t, \Pi_s \in \mathcal{F}_s,\Pi} & -\sum_{x\in\mathcal{X}} \sum_{y\in\mathcal{Y}_x} t_{xy}(\pi_{xy}^{t})P_t(x) - \sum_{y\in\mathcal{Y}} \sum_{x\in\mathcal{X}_y} s_{xy}(\pi_{xy}^{s})P_t(x) \\
\mathrm{s.t.}\quad & \pi_{xy}^s = \pi_{xy},\ \forall \{x,y\}\in\mathcal{E},\\
& \pi_{xy}^t = \pi_{xy},\ \forall \{x,y\}\in\mathcal{E},
\end{aligned}
\end{equation}
where 
\begin{align*}
\Pi_t:=&\{\pi_{xy}^t\}_{x\in\mathcal{X}_y,y\in\mathcal{Y}}, \quad
\Pi_s:=\{\pi_{xy}^s\}_{x\in\mathcal{X},y\in\mathcal{Y}_x}, \\ \mathcal{F}_t := & \left\{ \Pi_t | \pi_{xy}^t \geq 0, \underline{p}_x \leq \sum_{y \in \mathcal{Y}_x} \pi_{xy}^t \leq \bar{p}_x,\ \{x,y\} \in \mathcal{E}\right\}, \\ \mathcal{F}_s := & \left\{ \Pi_s | \pi_{xy}^s \geq 0, \underline{q}_y \leq \sum_{x \in \mathcal{X}_y} \pi_{xy}^s P_t(x)N \leq \bar{q}_y,\ \{x,y\} \in \mathcal{E} \right\}.
\end{align*}

To solve \eqref{OT2:eqn}, we leverage the alternating direction method of multipliers (ADMM). First, let $\alpha_{xy}^s$ and $\alpha_{xy}^t$ be the Lagrangian multipliers associated with the constraint $\pi_{xy}^s = \pi_{xy}$ and $\pi_{xy}^t = \pi_{xy}$, respectively. Then, the Lagrangian function associated with the optimization problem \eqref{OT2:eqn} can be written as follows:
\begin{equation}\label{Lag:eqn}
\begin{split}
    L &\left(\Pi_{t}, \Pi_{s}, \Pi, \alpha_{xy}^{t}, \alpha_{xy}^{s} \right) =\\
    &- \sum_{x\in\mathcal{X}} \sum_{y\in\mathcal{Y}_x} t_{xy}(\pi_{xy}^{t})P_t(x) - \sum_{y\in\mathcal{Y}} \sum_{x\in\mathcal{X}_y} s_{xy}(\pi_{xy}^{s})P_t(x) \\
    & +\sum_{x\in\mathcal{X}} \sum_{y\in\mathcal{Y}_x} \alpha_{xy}^{t} (\pi_{xy}^{t} - \pi_{xy}) \\ &+\sum_{y\in\mathcal{Y}} \sum_{x\in\mathcal{X}_y} \alpha_{xy}^{s} (\pi_{xy} - \pi_{xy}^{s}) + \frac{\eta}{2} \sum_{x\in\mathcal{X}} \sum_{y\in\mathcal{Y}_x} (\pi_{xy}^{t} - \pi_{xy})^2 \\
    &+ \frac{\eta}{2} \sum_{y\in\mathcal{Y}} \sum_{x\in\mathcal{X}_y} (\pi_{xy} - \pi_{xy}^{s})^2,
\end{split}
\end{equation}
where $\eta > 0$ is a positive scalar constant controlling the convergence rate in the algorithm designed below. Note that the last two terms in \eqref{Lag:eqn}, $\frac{\eta}{2} \sum_{x\in\mathcal{X}} \sum_{y\in\mathcal{Y}_x} (\pi_{xy}^{t} - \pi_{xy})^2$ and $\frac{\eta}{2} \sum_{y\in\mathcal{Y}} \sum_{x\in\mathcal{X}_y} (\pi_{xy} - \pi_{xy}^{s})^2$, acting as penalization, are quadratic. Hence, the Lagrangian function $L$ is strictly convex, ensuring the existence of a unique optimal solution. 

We can apply ADMM to the minimization problem in \eqref{OT2:eqn}. The designed distributed algorithm is presented in the following proposition.

\begin{proposition}
 The iterative steps of ADMM to \eqref{OT2:eqn} are summarized as follows:
\begin{equation}\label{ADMM1_eqn1}
\begin{split}
    \Pi_{x,t}(k+1) \in \arg \min_{\Pi_{x,t}\in\mathcal{F}_{x,t}} - \sum_{y\in\mathcal{Y}_x} t_{xy}(\pi_{xy}^t)P_t(x) \\ + \sum_{y\in\mathcal{Y}_x} \alpha_{xy}^t(k) \pi_{xy}^t + \frac{\eta}{2} \sum_{y\in\mathcal{Y}_x} (\pi_{xy}^t - \pi_{xy}(k))^2,
\end{split}
\end{equation}
\begin{equation}\label{ADMM1_eqn2}
\begin{aligned}
        \Pi_{y,s}(k+1) \in \arg \min_{\Pi_{y,s}\in\mathcal{F}_{y,s}} - \sum_{x\in\mathcal{X}_y} s_{xy}(\pi_{xy}^s)P_t(x)  \\ -\sum_{x\in\mathcal{X}_y} \alpha_{xy}^s(k)\pi_{xy}^s + \frac{\eta}{2} \sum_{x\in\mathcal{X}_y} (\pi_{xy}(k) - \pi_{xy}^s)^2,
\end{aligned}
\end{equation}
\begin{equation}\label{ADMM1_eqn3}
\begin{split}
    \pi_{xy}(&k+1)= \arg \min_{\pi_{xy}} - \alpha_{xy}^t(k)\pi_{xy} + \alpha_{xy}^s(k)\pi_{xy} \\
    &+\frac{\eta}{2}(\pi_{xy}^t(k+1) - \pi_{xy})^2 + \frac{\eta}{2}(\pi_{xy} - \pi_{xy}^s(k+1))^2,
\end{split}
\end{equation}
\begin{equation}\label{ADMM1_eqn4}
\begin{split}
    \alpha_{xy}^t(k+1) = \alpha_{xy}^t(k) + \eta(\pi_{xy}^t(k+1)-\pi_{xy}(k+1))^2,
\end{split}
\end{equation}
\begin{equation}\label{ADMM1_eqn5}
\begin{split}
    \alpha_{xy}^s(k+1) = \alpha_{xy}^s(k) + \eta(\pi_{xy}(k+1)-\pi_{xy}^s(k+1))^2,
\end{split}
\end{equation}
where $\Pi_{\tilde{x},t}:=\{\pi_{xy}^t\}_{y\in\mathcal{Y}_x,x=\tilde{x}}$ represents the solution at target node $\tilde{x}\in\mathcal{X}$, and $\Pi_{\tilde{y},s}:=\{\pi_{xy}^s\}_{x\in\mathcal{X}_y,y=\tilde{y}}$ represents the proposed solution at source node $\tilde{y}\in\mathcal{Y}$. In addition, $\mathcal{F}_{x,t} := \{ \Pi_{x,t} | \pi_{xy}^t \geq 0, y\in\mathcal{Y}_x, \underline{p}_x \leq \sum_{y \in \mathcal{Y}_x} \pi_{xy}^t \leq \bar{p}_x\}$, and $\mathcal{F}_{y,s} := \{ \Pi_{y,s} | \pi_{xy}^s \geq 0, x\in\mathcal{X}_y, \underline{q}_y \leq \sum_{x \in \mathcal{X}_y} \pi_{xy}^sP_t(x)N \leq \bar{q}_y\}$.
\end{proposition}

\begin{proof}
Let $\overrightarrow{x}=[\overrightarrow{\Pi}_{x,t}^T,\overrightarrow{\Pi}^T]^T, \overrightarrow{y} = [\overrightarrow{\Pi}^T,\overrightarrow{\Pi}_{y,s}^T]$, and $\alpha=[{\alpha_{xy}^t}^T,{\alpha_{xy}^s}^T]^T$ where $\overrightarrow{\cdot}$ denotes the vectorization operator and T denotes the transpose operator. All three vectors are of $2|\mathcal{E}| \times 1$ size. We write the constraints in \eqref{OT2:eqn} in a matrix form such that $\textbf{A}\overrightarrow{x} = \overrightarrow{y}$, where $\textbf{A} = [\textbf{I},\textbf{0};\textbf{0},\textbf{I}]$ with $\textbf{I}$ and $\textbf{0}$ denoting the $| \mathcal{E}|$-dimensional identity and zero matrices, respectively. We also have $\overrightarrow{x} \in \mathcal{F}_{\overrightarrow{x}}^{t}$ and $\overrightarrow{y} \in \mathcal{F}_{\overrightarrow{y}}^{s}$ where \begin{align*}
\mathcal{F}_{\overrightarrow{x}}^{t} := &\{\overrightarrow{x} | \pi_{xy}^{t}, \underline{p}_{x} \leq \sum_{y \in \mathcal{Y}_{x}} \pi^{t}_{xy} \leq \bar{p}_{x}, \{x,y\} \in \mathcal{E}\}
\end{align*}
and 
\begin{align*}
\mathcal{F}_{\overrightarrow{y}}^{s}:=&\{\overrightarrow{y} | \pi_{xy}^{s}, \underline{q}_{y} \leq \sum_{x \in \mathcal{X}_{y}} \pi_{xy}^{s} \leq \bar{q_{y}}, \{x,y\} \in \mathcal{E}.
\end{align*}
Then, \eqref{OT2:eqn} can be solved using the following iterations: 
\begin{itemize}
\item[(1)] $\overrightarrow{x}(k+1) \in \arg \min_{\overrightarrow{x} \in \mathcal{F}_{\overrightarrow{x}}^{t}} L(\overrightarrow{x}, \overrightarrow{y}(k), \alpha(k))$; 
\item[(2)] $\overrightarrow{y}(k+1) \in \arg \min_{\overrightarrow{y} \in \mathcal{F}_{\overrightarrow{y}}^{s}} L(\overrightarrow{x}(k+1), \overrightarrow{y}, \alpha(k))$;\item[(3)] $\alpha(k+1) = \alpha(k) +  \eta(A\overrightarrow{x}(k+1) - \overrightarrow{y}(k+1))$, based on \cite{10.1561/2200000016}. 
\end{itemize}
Since there are no couplings among $\Pi_{x}^{t}, \Pi_{y}^{s}, \alpha^{t}_{xy},$ and $\alpha^{s}_{xy}$, these iterations are equivalent to \eqref{ADMM1_eqn1} - \eqref{ADMM1_eqn5}.
\end{proof}

 We can further simplify \eqref{ADMM1_eqn1}-\eqref{ADMM1_eqn5}, and the results are summarized below.

\begin{proposition}\label{prop:2}
The steps \eqref{ADMM1_eqn1}-\eqref{ADMM1_eqn5} can be simplified as follows:
\begin{equation}\label{ADMM2_eqn1}
\begin{split}
    \Pi_{x,t}(k+1) \in \arg \min_{\Pi_{x,t}\in\mathcal{F}_{x,t}} - \sum_{y\in\mathcal{Y}_x} t_{xy}(\pi_{xy}^{t})P_t(x) \\
    + \sum_{y\in\mathcal{Y}_x} \alpha_{xy}(k) \pi_{xy}^{t} + \frac{\eta}{2} \sum_{y\in\mathcal{Y}_x} \left(\pi_{xy}^{t} - \pi_{xy}(k)\right)^2,
\end{split}
\end{equation}
\begin{equation}\label{ADMM2_eqn2}
\begin{split}
    \Pi_{y,s}(k+1) \in \arg \min_{\Pi_{y,s}\in\mathcal{F}_{y,s}} - \sum_{x\in\mathcal{X}_y} s_{xy}(\pi_{xy}^{s})P_t(x) \\ -\sum_{x\in\mathcal{X}_y} \alpha_{xy}(k)\pi_{xy}^{s} + \frac{\eta}{2} \sum_{x\in\mathcal{X}_y} \left(\pi_{xy}(k) - \pi_{xy}^{s}\right)^2,
\end{split}
\end{equation}
\begin{equation}\label{ADMM2_eqn3}
\begin{split}
    \pi_{xy}(k+1) = \frac{1}{2} \left(\pi_{xy}^{t}(k+1) + \pi_{xy}^{s}(k+1)\right),
\end{split}
\end{equation}
\begin{equation}\label{ADMM2_eqn4}
\begin{split}
    \alpha_{xy}(k+1) = \alpha_{xy}(k) + \frac{\eta}{2}\left(\pi_{xy}^{t}(k+1) - \pi_{xy}^{s}(k+1)\right).
\end{split}
\end{equation}
\end{proposition}

\begin{proof}
As \eqref{ADMM1_eqn1} is strictly concave, we can solve it by first-order condition: $$\pi_{xy}(k+1) = \frac{1}{2\eta}(\alpha_{xy}^{t}(k) - \alpha_{xy}^{s}(k)) + \frac{1}{2}(\pi_{xy}^{t}(k+1) + \pi_{xy}^{s}(k+1)).$$ By substituting this equation into \eqref{ADMM1_eqn4} and \eqref{ADMM1_eqn5} we get: $$\alpha_{xy}^{t}(k+1) = \frac{1}{2}(\alpha_{xy}^{t}(k) + \alpha_{xy}^{s}(k)) + \frac{\eta}{2}(\pi_{xy}^{t}(k+1) - \pi_{xy}^{s}(k+1)),$$ $$\alpha_{xy}^{s}(k+1) = \frac{1}{2}(\alpha_{xy}^{t}(k) + \alpha_{xy}^{s}(k)) + \frac{\eta}{2}(\pi_{xy}^{t}(k+1) - \pi_{xy}^{s}(k+1)).$$ We see that $\alpha_{xy}^{t} = \alpha_{xy}^{s}$ during each update. Then, $\pi_{xy}(k+1)$ can be further simplified as $$\pi_{xy}(k+1) = \frac{1}{2}(\pi_{xy}^{t}(k+1) + \pi_{xy}^{s}(k+1))$$ shown in \eqref{ADMM2_eqn3}. In addition, we get \eqref{ADMM2_eqn4} from $\alpha_{xy}^{t}=\alpha_{xy}^s=\alpha_{xy}$ in \eqref{ADMM1_eqn4} and \eqref{ADMM1_eqn5}.
\end{proof}

We summarize the steps into Algorithm \ref{Alg:1}.

\begin{algorithm}[!t]
\caption{Distributed Algorithm}\label{Alg:1}
\begin{algorithmic}[1]
\While {$\Pi_{x,t}$ and $\Pi_{y,s}$ not converging}
\State Compute $\Pi_{x,t}(k+1)$  using \eqref{ADMM2_eqn1}, for all $x\in\mathcal{X}_y$
\State Compute $\Pi_{y,s}(k+1)$  using \eqref{ADMM2_eqn2}, for all $y\in\mathcal{Y}_x$
\State Compute $\pi_{xy}(k+1)$  using \eqref{ADMM2_eqn3}, for all $\{x,y\}\in \mathcal{E}$
\State Compute $\alpha_{xy}(k+1)$  using \eqref{ADMM2_eqn4}, for all $\{x,y\}\in \mathcal{E}$
\EndWhile
\State \textbf{return} $\pi_{xy}(k+1)$, for all $\{x,y\}\in \mathcal{E}$
\end{algorithmic}
\end{algorithm}

\section{Federated Learning under Unknown Type Distribution}\label{sec:OT_unknown}

This section aims to design an efficient optimal transport mechanism when the type distribution of target nodes $P_t(x)$ is unknown. In this scenario, the transport planner does not have complete information on the resource receivers due to its large population feature or too expensive computationally to extract the exact information. Instead, statistical update on the type distribution will be provided by sequentially revealed data on the targets nodes. 
To this end, we need to develop a data-driven method to counteract the unknown target's type information, which is a federated learning algorithm with a convergence guarantee to compute the optimal transport plan.

Specifically, although the type distribution is fixed, but is unknown to the transport planer. At each step of the algorithm, we assume that the type for one of the nodes will be identified according to the unknown distribution. For $k=1,2,\ldots$, $x(k)$ denotes the node type revealed at $k$-th step of the algorithm. For the algorithm, at $k$-th step, the $x(k)$-th type target nodes will update their transport decision.

Denote $\Pi$ the transport plan before the update. Define the prox operator based on $\mu>0$, $\Pi_0$ and $x\in \mathcal{X}$ as, 
\begin{equation}\label{prox_z_mu}
\begin{split}
    z_\mu(\Pi_0;x) = {\arg\min}_{\Pi}\ -\sum_{y\in\mathcal{Y}_{x}} t_{{x}y}(\pi_{{x}y}) - \sum_{y\in\mathcal{Y}_{x_i}} s_{{x_i}y}(\pi_{{x_i}y}) \\
    +\frac{1}{2\mu}\Vert \Pi - \Pi_0\Vert^2.
\end{split}
\end{equation}
Evidently, $-\sum_{y\in\mathcal{Y}_{x_i}} t_{{x_i}y}(\pi_{{x_i}y}) - \sum_{y\in\mathcal{Y}_{x_i}} s_{{x_i}y}(\pi_{{x_i}y})$ represents the cost related to $x_i$ and $\frac{1}{2\mu}\Vert \Vert \Pi - \Pi_0\Vert^2$ the square penalty for change.

At step $k$, with existing transport plan $\Pi(k)$ and pre-specified learning rate $\mu_k$, local planner at target node of type $x(k)$ addresses the following problem locally:
\begin{equation}\label{proximal_1}
    \begin{split}
        z_{\mu_k}(\Pi(k);x(k) = x)= {\arg\min}_{\Pi}\ -\sum_{y\in\mathcal{Y}_{x}} t_{{x}y}(\pi_{{x}y}) \\
    - \sum_{y\in\mathcal{Y}_{x}}s_{{x}y}(\pi_{{x}y})+\frac{1}{2\mu_k}\Vert \Pi - \Pi(k)\Vert^2.
    \end{split}
\end{equation}

Then, the local solution $z_{\mu_k}(\Pi(k);x(k) = x)$ that contains only the transport plans from the sources to $x(k)$, computed by the target node is sent to the central planner who will project the solution to the feasible space and broadcast it to the entire transport network. The projection step executed by the central planner is as follows:
\begin{equation}\label{projection_1}
\begin{split}
   \Pi(k+1) 
   = \mathrm{Proj}_{\mathcal{L}^k}( z_{\mu_k}(\Pi(k);x(k) = x_i);\Pi_{-x_i}),
   \end{split}
\end{equation}
where the feasible set $\mathcal{L}^k$ is determined by
$$
\mathcal{L}^k = \{\Pi | \underline{p}_{x}\leq \sum_{y\in\mathcal{Y}_x} \pi_{xy}\leq \bar{p}_{x}, \forall x\in\mathcal{X}, $$

$$\underline{q}_{y}\leq \sum_{x\in\mathcal{X}_y} \pi_{xy}\tilde{P}_t^k(x)N\leq \bar{q}_{y}, \forall y\in\mathcal{Y}, $$ $$\pi_{xy}\geq 0, \forall \{x,y\} \in\mathcal{E}\};$$
with
$$\tilde{P}_t^k(x) = \sum_{j=1}^k\frac{\mathbf{1}_{x(j) = x}}{k}$$ 
denoting the empirical type distribution of target nodes up to type $k$ and $\Pi_{-x_i}$ denoting the solution of all types of targets except type $x_i$.

The developed federated learning algorithm is summarized in Algorithm \ref{algorithm_prox}.

\begin{algorithm}[!t]
\caption{Federated Learning Algorithm}\label{algorithm_prox}
\begin{algorithmic}[1]
\For {$k=1,2,...,N$}
\State Sample the target node's type data $x(k)$ from the unknown distribution $P_t(x)$
\State If $x(k) = x_i$, target nodes of type $x_i$ do the updates:
\State \quad Compute $z_{\mu_k}(\Pi(k);x(k) = x_i)$ based on \eqref{proximal_1}
\State \quad Obtain $\Pi(k+1)$ by projection step \eqref{projection_1}
\State\quad Broadcast $\Pi(k+1)$ to the transport network
\EndFor
\State \textbf{return} $\Pi(N)$
\end{algorithmic}
\end{algorithm}

\begin{remark} The total population parameter $N$ does not need to be accurately known at the beginning. The reason is that $N$ is a constant in the objective function in \eqref{OT1:eqn} and thus can be excluded. The other place $N$ plays a role is in the second constraint in \eqref{OT1:eqn} which affects the project step \eqref{projection_1}. During the learning, the central planner can first have an estimate of $N$, say $\tilde{N}$, which might be an overestimate. The final learned transport plan could be easily scaled with a ratio $N/\tilde{N}$ to reconstruct the optimal strategy when the learning process finishes.
\end{remark}

\subsection{Convergence Proof of Algorithm \ref{algorithm_prox}}

We show that $z_{\mu}(\Pi_x;x_i)$ is a smooth convex function under proper assumptions.  and denote:

$$
    f(\Pi;x_i)=-\sum_{y\in\mathcal{Y}_{x_i}} t_{{x_i}y}(\pi_{{x_i}y}) - \sum_{y\in\mathcal{Y}_{x_i}} s_{{x_i}y}(\pi_{{x_i}y}),
$$

\begin{align*}
f_{\mu}(\Pi;x_{i}):=f(z_{\mu}(\Pi;x_{i});x_{i}) 
+ \frac{1}{2\mu}\|z_{\mu}(\Pi;x_{i})-\Pi\|^2.
\end{align*}

For the consequence analysis, we have two extra assumptions on the utility functions $t_{xy}$ and $s_{xy}$. 
\begin{assumption}\label{assumption:assum2}
The utility functions $t_{xy}$ and $s_{xy}$ is proper and Lipschitz continuous, and let $L_t$ and $L_s$ denote their Lipschitz constant, that is, for any $m,n \in \mathbb{R}^{n}$, we have, $\|s_{xy}(m)-s_{xy}(n)\|\le L_s\|m-n\|$ and $\|t_{xy}(m)-t_{xy}(n)\|\le L_t\|m-n\|$. 

Therefore, for any $x_{i} \in \mathcal{X}$, we have, $|f(m;x_{i}) - f(n;x_{i})| \leq (L_s+L_t)||m - n||,  \forall m,n \in \mathbb{R}^{n}$.
\end{assumption}

\begin{assumption}\label{assumption:assum3}
The set, $\mathcal{L}^k$, onto which the solutions are projected, are simple convex sets, and there exists $\xi >0$ such that the feasible set $\mathcal{X}$ satisfies linear regularity:
$dist^2_{\mathcal{X}}(x) \leq \xi \mathbb{E}[dist^2_{\mathcal{L}_{k}}(x)]$
\end{assumption}

\begin{lemma}\label{lemma:lemma1}
Let $\mu>0, x_{i} \in \mathcal{X}$.
Then, for any $g_f(\Pi;x_{i}) \in \delta f(\Pi;x_i)$, the subdifferential set,  the following holds: 

$$||\nabla f_{\mu}(\Pi;x_i)|| \leq ||g_{f}(\Pi;x_i)||, \quad \forall \Pi_{x_i} \in \mathbb{R}^{n}.
 $$
 \end{lemma}

 \begin{proof}
See Appendix \ref{app:lemma_1}. 
\end{proof}

We define $\hat{\mu}_{1,k} = \sum_{i=0}^{k-1}\mu_i$  and $\hat{\mu}_{2,k} = \sum_{i=0}^{k-1} \mu_i^2 $. Let $F_{k}=\sigma\{x(0),x(1),...,x(k)\}$ be the $\sigma$-algebra by the first $k$ data variables, representing the history of random choices $x(k)$.  Let $\hat{\Pi}(k) = \frac{1}{\hat{\mu}_{1,k}}\sum_{i=0}^{k-1}\mu_i\Pi(i)$ be the transport plan after the projection, and $\hat{\Lambda}(k) = \frac{1}{\hat{\mu}_{1,k}}\sum_{i=0}^{k-1}\mu_i\Lambda(i)$ be the averaged sequences, where $\Lambda(k) = z_{\mu}(\Pi_x;x_i)$, the local solution before the projection. Also note that $\mathbb{E}$ represents the expectation.

\begin{lemma}\label{lemma:lemma2}
With Assumptions 2 and 3 holding true, let the random sequences $\{\Pi(k),\Lambda(k)\}_{k\geq0}$ be generated with nonincreasing positive step-sizes $\{\mu_k\}$, then the following relation holds:
\begin{align*}
&\mathbb{E}[dist^2(\hat{\Lambda}(k), \mathcal{L}^{k})]\\ \ge& \frac{1}{2\xi}dist^2(\hat{\Pi}(k), \mathcal{X})]-\frac{\hat{\mu}_{2,k}(L_s+L_t)}{2}\mathbb{E}[dist(\hat{\Pi}(k),\mathcal{X})].
\end{align*}
\end{lemma}

\begin{proof}
By convexity of $G_{\mu,x_{i}} = \frac{1}{2\mu}dist^2(x,\mathcal{L}^{k})$ and conditional expectation of $x(k)$ w.r.t $F_{k-1}=\{x(0),x(1),...,x(k-1)\}$, we get:
\begin{align*}
    \mathbb{E}[G_{\mu,x(k)}(\hat{\Lambda}(k))|F_{k-1}] \geq & \mathbb{E}[G_{\mu,x(k)}(\hat{\Pi}(k))
    \\ &+\langle \nabla G_{\mu,x(k)}(\hat{\Pi}(k)), \hat{\Lambda}(k) - \hat{\Pi}(k) \rangle|F_{k-1}].
\end{align*}
Then, taking expectation over $F_{k-1}$ gives us the following inequality:
\begin{align*}
    &\mathbb{E}[G_{\mu,x(k)}(\hat{\Lambda}(k)] \\\geq &\mathbb{E}[G_{\mu,x(k)}(\hat{\Pi}(k))]
    + \mathbb{E}[\langle \nabla
    G_{\mu,x(k)}(\hat{\Pi}(k)),\hat{\Lambda}(k)-\hat{\Pi}(k) \rangle]\\
    =&\mathbb{E}[G_{\mu,x(k)}(\hat{\Pi}(k))]
    + \mathbb{E}[\langle \nabla
    G_{\mu,x(k)}(\hat{\Pi}(k)),\frac{1}{\hat{\mu}_{1,k}}\sum_{i=0}^{k-1}\mu_i[\Pi(i)-\Lambda(i)] \rangle]\\
    =&\mathbb{E}[G_{\mu,x(k)}(\hat{\Pi}(k))]
    \\ &+ \mathbb{E}[\langle \nabla
    G_{\mu,x(k)}(\hat{\Pi}(k)),\frac{1}{\hat{\mu}_{1,k}}\sum_{i=0}^{k-1}\mu_i[\Pi(i)-z_{\mu_i}(\Pi(i), x(i))] \rangle]
\end{align*}
Using Cauchy-Schwarz inequality gives
\begin{align*}
&\mathbb{E}[G_{\mu,x(k)}(\hat{\Pi}(k))]
    \\ &+ \mathbb{E}[\langle \nabla
    G_{\mu,x(k)}(\hat{\Pi}(k)),\frac{1}{\mu}\sum_{i=0}^{k-1}\mu_i[\Pi(i)-z_{\mu_i}(\Pi(i), x(i))] \rangle]\\
&\mathbb{E}[G_{\mu,x(k)}(\hat{\Pi}(k))]
    \\ &-\frac{1}{\hat{\mu}_{1,k}}\mathbb{E}[\| \nabla
    G_{\mu,x(k)}(\hat{\Pi}(k))\mathbb{E}[\sum_{i=0}^{k-1}\mu_i\|\Pi(i)-z_{\mu_i}(\Pi(i), x(i))\|].
\end{align*}
Optimality condition for problem~\eqref{proximal_1} indicates, 
\begin{align*}
&\mathbb{E}[G_{\mu,x(k)}(\hat{\Pi}(k))]
    \\ &-\frac{1}{\hat{\mu}_{1,k}}\mathbb{E}[\| \nabla
    G_{\mu,x(k)}(\hat{\Pi}(k))\mathbb{E}\left[\sum_{i=0}^{k-1}\mu_i\|\Pi(i)-z_{\mu_i}(\Pi(i), x(i))\|\right]\\
=&\mathbb{E}[G_{\mu,x(k)}(\hat{\Pi}(k))]
    \\ &-\frac{1}{\hat{\mu}_{1,k}}\mathbb{E}[\| \nabla
    G_{\mu,x(k)}(\hat{\Pi}(k))\mathbb{E}\left[\sum_{i=0}^{k-1}\mu_i^2\nabla f(\Pi(i), x(i))\right]  
    .
\end{align*}
Lipschitz condition then implies that,
\begin{align*}
     &\mathbb{E}[G_{\mu,x(k)}(\hat{\Lambda}(k)]  \geq \mathbb{E}[G_{\mu,x(k)}(\hat{\Pi}(k))]
    - \frac{\hat{\mu}_{2,k}(L_s+L_t)}{\hat{\mu}_{1,k}}\sqrt{\mathbb{E}[dist^2_{\mathcal{X}}(\hat{\Pi}(k))]}.
\end{align*}
Therefore, for $\mu= \hat{\mu}_{1,k}$, we have,
\begin{align*}
&\mathbb{E}[dist^2(\hat{\Lambda}(k), \mathcal{L}^{k})]\\ \ge& \frac{1}{2\xi}dist^2(\hat{\Pi}(k), \mathcal{X})]-\frac{\hat{\mu}_{2,k}(L_s+L_t)}{2}\mathbb{E}[dist(\hat{\Pi}(k),\mathcal{X})].
\end{align*}
This is the desired inequality. 
\end{proof}
Lemma~\ref{lemma:lemma2} provides a quantitative estimate of the descent for the stochastic convex optimization problem, with the constraints present. With Lemma~\ref{lemma:lemma2}, we proceed to our main result on the estimations of suboptimality and feasibility violation. 

\begin{theorem}\label{thm:1}
Using Assumption 2 and 3, let sequence $\{\Pi(k)\}_{k \ge 0}$ be generated by the federated learning algorithm with nonincreasing positive step-sizes $\{\mu_k\}$ and define $R_k=\hat{\mu}_{1,k} \xi (\|\Pi_{0} - \Pi^{*}\|^{2} + k(L_s+L_t)\hat{\mu}_{2,k})$, where $\Pi^{*}$ is the optimal set of solutions and $\Pi^{0}$ is the starting point. Then, the following are the constraints for the suboptimality and feasibility violation:
\begin{align*}
\begin{split}
    -3\xi(L_s+L_t)\hat{\mu}_{1,k} - (L_s+L_t)\sqrt{\frac{R_{k}}{k\hat{\mu}_{1,k}}} \\
    \leq \mathbb{E}[F(\hat{\Pi}(k))]-F^{*}
    \leq \frac{R_{k}}{2k\xi\hat{\mu}_{2,k}},
\end{split}
\end{align*}
and 
\begin{align*}
    \mathbb{E}[dist^{2}(\hat\Pi(k),\mathcal{X})] \leq 2\xi^2(L_s+L_t)^2(3\hat{\mu}_{1,k})^2 + \frac{2R_k}{k\hat{\mu}_{1,k}}.
\end{align*}
\end{theorem}

\begin{proof}
See Appendix \ref{app:thm_1}. 
\end{proof}

Corollary 1 below shows us that our solution generated by the federated learning algorithm converges. 

\begin{corollary}
Under the assumptions of Theorem 1, let $\{\Pi(k)\}_{k \geq 0}$ be the sequence generated by the federated learning algorithm with constant step size $\{\mu > 0\}$. Let $\epsilon > 0$ be the desired accuracy and $K$ be an integer satisfying:
\begin{align*}
    K \geq \frac{(L_s+L_t)^2\|\Pi^{0}-\Pi^{*}\|^2}{\epsilon^2} \max\{1,(3\xi+\sqrt{2\xi})^2\},
\end{align*}
and the step size be chosen as
\begin{align*}
    \mu = \frac{\epsilon}{(L_s+L_t)^2(3\xi+\sqrt{2\xi})}.
\end{align*}
Then, after $K$ iterations, the average point $\hat{\Pi}(K) = \frac{1}{K} \sum_{i=0}^{K-1} \Pi(i)$ satisfies:
\begin{align*}
\begin{split}
    \lvert \mathbb{E}[F(\hat{\Pi}(K)] - F^{*} \rvert \leq \epsilon \mbox{ and }
    \sqrt{\mathbb{E}[dist^{2}(\hat\Pi(k),\mathcal{X})]} \leq \epsilon.
    \end{split}
\end{align*}
\end{corollary}

\begin{proof}
Let $k = K$ in Theorem 1. We want to obtain the lowest value of the right-hand side of the inequalities in Theorem 1 by choosing $\mu_i=\mu$ and minimizing $\mu$. Let $r_0 = \|\Pi^{0} - \Pi^{*}\|$, we get the following optimal smoothing parameter: 
\begin{align*}
    \mu = \sqrt{\frac{r_{0}^2}{K(L_s+L_t)^2}}
\end{align*}
for the optimal rate:
\begin{align*}
    \mathbb{E} [F(\hat{\Pi}(K))] - F^{*} \leq \sqrt{\frac{(L_s+L_t)^2r_{0}^2}{K}}
\end{align*}
Then, by using the optimal parameter $\mu$ and the relations of Theorem 1, we get
\begin{align*}
    \mathbb{E}[dist^{2}(\hat\Pi(k),\mathcal{X})] \leq \frac{r_{0}^2}{K}(18\xi^2  + r\xi)
\end{align*}
which is the feasibility bound and 
\begin{align*}
    \mathbb{E}[F(\hat{\Pi}(K))] - F^{*} \geq -(3\xi + \sqrt{2\xi}) \sqrt{\frac{(L_s+L_t)^2r_{0}^2}{K}}
\end{align*}
which is the lower suboptimality bound. Combining all the results gives us 
\begin{align*}
    K \geq \frac{(L_s+L_t)^2\|\Pi^{0}-\Pi^{*}\|^2}{\epsilon^2} \max\{1,(3\xi+\sqrt{2\xi})^2\}.
\end{align*}
\end{proof}
Thus, with a constant $\mu$ and a decent amount of iterations of the federated learning algorithm, we arrive at an optimal solution.

\section{Case Studies}\label{sec:cases}
In this section, we evaluate the performance of our federated learning algorithm through two distinct case studies. The first case study assesses the efficiency of the algorithm, and the second case study explores the algorithm’s resilience to changes during the learning process. In both case studies, we consider a scenario with three types of target nodes
and two source nodes. The transport network can be represented as a complete bipartite network graph. The utility functions are defined as: $t_{xy}(\pi_{xy}) = \delta_{xy}\pi_{xy}$ and $s_{xy}(\pi_{xy}) = \gamma_{xy}\pi_{xy}$, where
\begin{align*}
    [\gamma_{xy}]_{x \in \mathcal{X}, y \in \mathcal{Y}} = \begin{bmatrix}
2 & 2 \\
3 & 2 \\
1 & 4
\end{bmatrix}, 
[\delta_{xy}]_{x \in \mathcal{X}, y \in \mathcal{Y}} = \begin{bmatrix}
2 & 4 \\
2 & 2 \\
4 & 4
\end{bmatrix}.
\end{align*}
The upper bound of resources received by three types of targets is $\bar{p} = [2,3,4]$ and $\bar{q} = [4,4]*300$. The lower bounds, $\underline{p}$ and $\underline{q}$ are 0 for all nodes. The learning rate is defined as $\mu = \frac{0.5}{\sqrt{i}}$, where $i$ represents the iteration step. Furthermore, we consider $N=8000$ target nodes with a type distribution of $P_{t} = [0.5, 0.3, 0.2]$.

\subsection{Distributed Resource Allocation Efficiency}

\begin{figure}[!ht]
\vspace{-1mm}
  \centering
  \subfigure[Empirical PDF of Target Nodes]{
    \includegraphics[width=0.45\columnwidth]{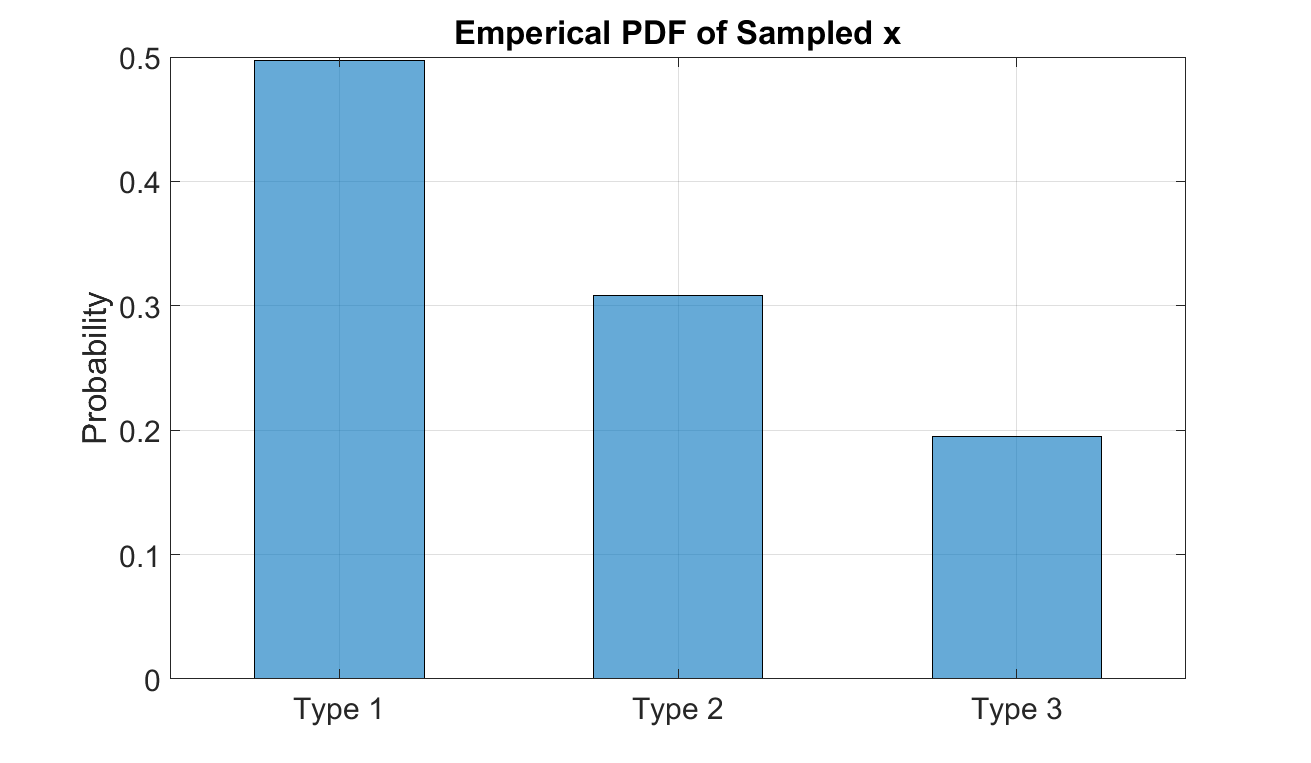}\label{fig:pdf}}
	 \subfigure[Transported Plan]{
    \includegraphics[width=0.45\columnwidth]{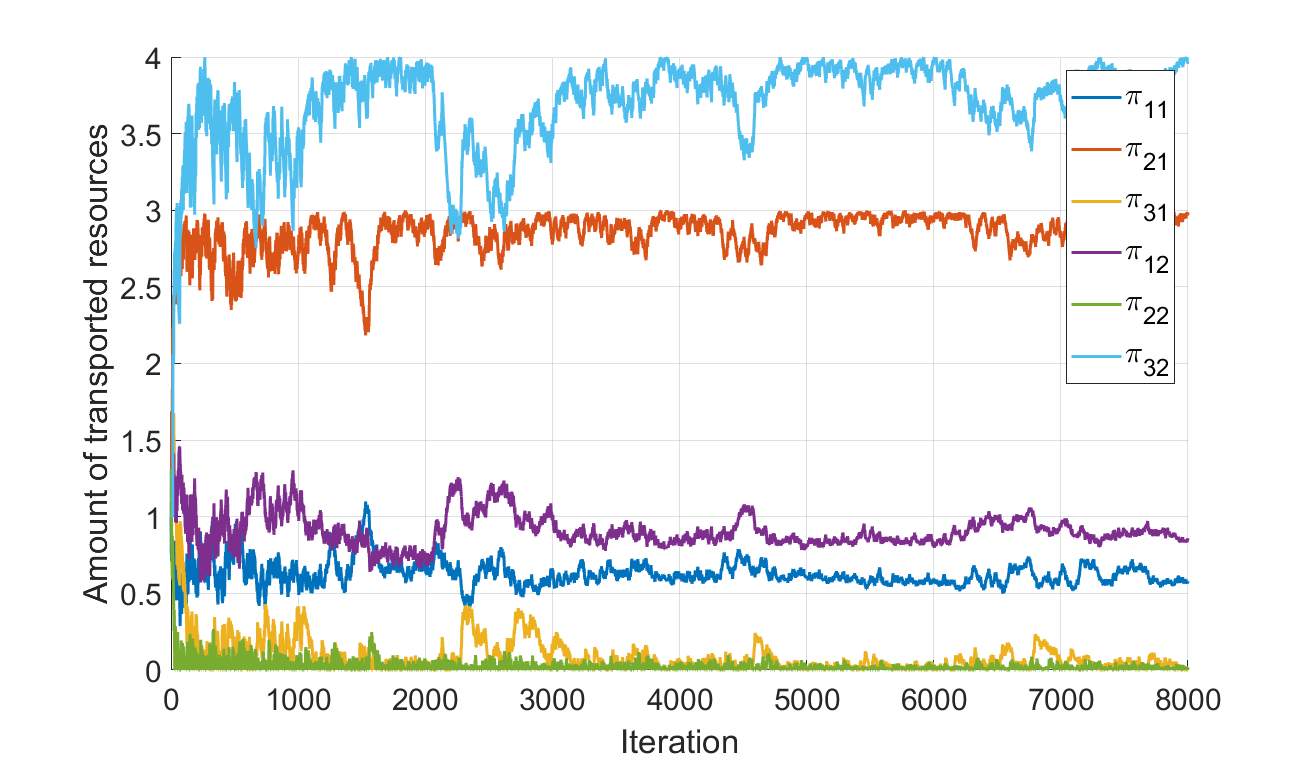}\label{fig:resources}}
    \subfigure[Transport Utility]{
    \includegraphics[width=0.45\columnwidth]{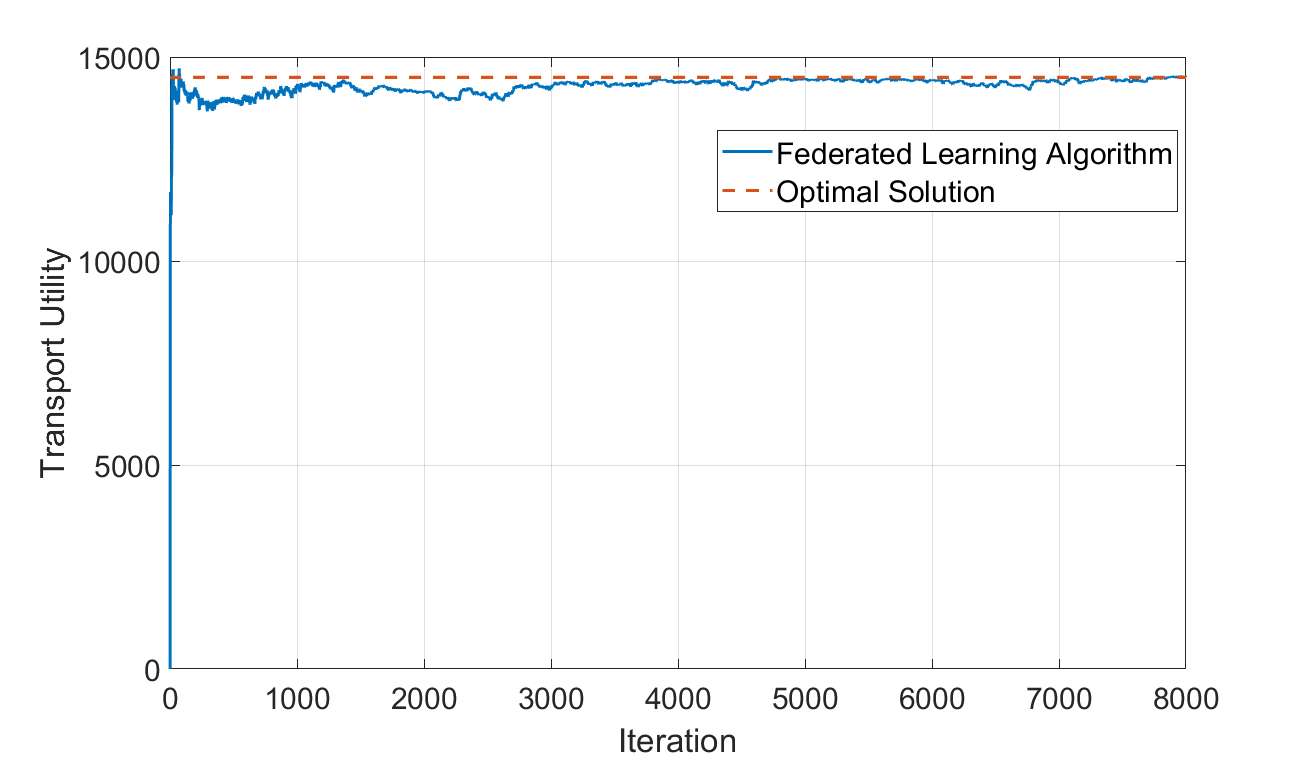}\label{fig:utility}}
    \subfigure[Received Resources at Targets]{
    \includegraphics[width=0.45\columnwidth]{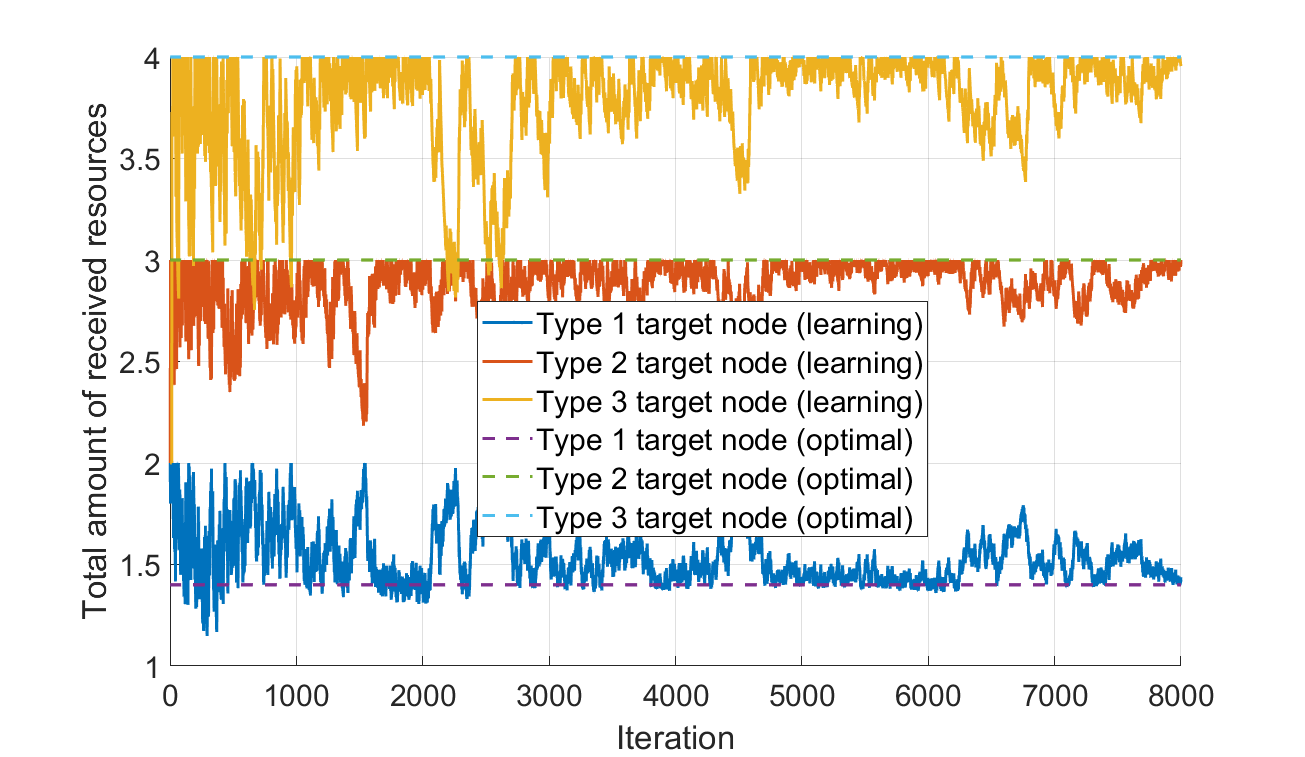}\label{fig:recresources}}\vspace{-1mm}
  \caption[]{(a): Distribution of the sampled targets' type data; (b): Transport plan; (c): Transport utility given by the federated learning algorithm; (d): Resource received by each type of target node.}
  \label{fig:AlgorithmResults}
\end{figure}

We show the efficiency of the designed federated learning algorithm by running Algorithm 2 for 8000 iterations ($N=8000$). The results are displayed in Figure \ref{fig:AlgorithmResults}. Figure \ref{fig:pdf} illustrates the PDF of the sampled types. The three target types—labeled as Type 1, Type 2, and Type 3—correspond to the type distribution previously described. The transport utility, which serves as the key metric for evaluating the algorithm’s efficiency, is shown in \ref{fig:utility}. By the end of the 8000 iterations, the utility values indicate that the federated learning algorithm successfully converged to the centralized optimal solution. Figure \ref{fig:resources} illustrates the corresponding resource allocation across the transport network, showing the distribution of resources from source nodes to the different target types. Finally, Figure \ref{fig:recresources} demonstrates that the total amount of resources received by each target type approaches the optimal allocation, confirming the algorithm's effectiveness in balancing resource demands across the network.

\subsection{Resilience in the Federated Learning Algorithm}
The second case study investigates the algorithm's ability to adapt to a dynamic environment. In real-world scenarios, changes in the network, such as fluctuations in the number of targets, shifts in type distribution, or varying preferences, are common. To simulate this, we introduce a shift in the target type distribution after 600 iterations, changing from $P_{t} = [0.50, 0.30, 0.20]$ to $P_{t} = [0.12, 0.65, 0.23]$. Figure \ref{fig:ResilienceAlgorithmResults} presents the results, showcasing how the algorithm adjusts to this new distribution and re-optimizes resource allocation in response to the changes.

\begin{figure}[!ht]
\vspace{-1mm}
  \centering
  \subfigure[Empirical PDF of Target Nodes]{
    \includegraphics[width=0.45\columnwidth]{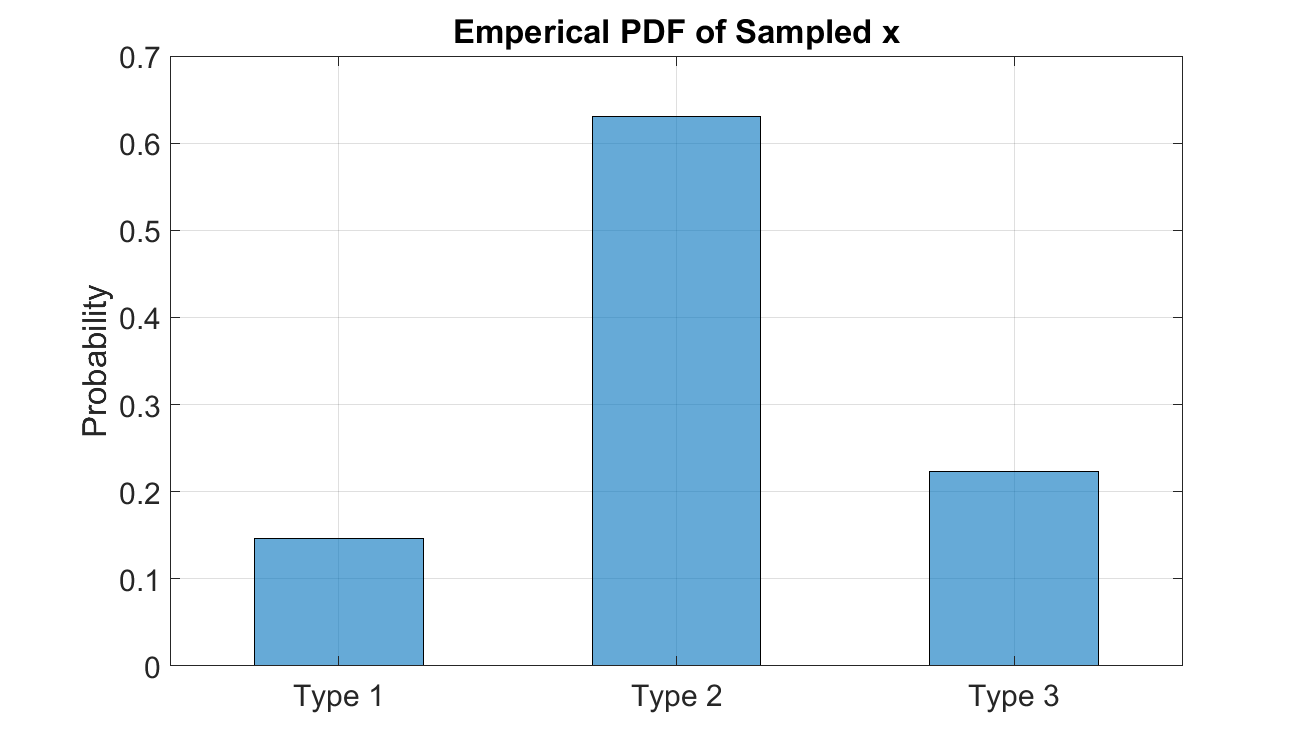}\label{fig:pdf1}}
	 \subfigure[Transported Plan]{
    \includegraphics[width=0.45\columnwidth]{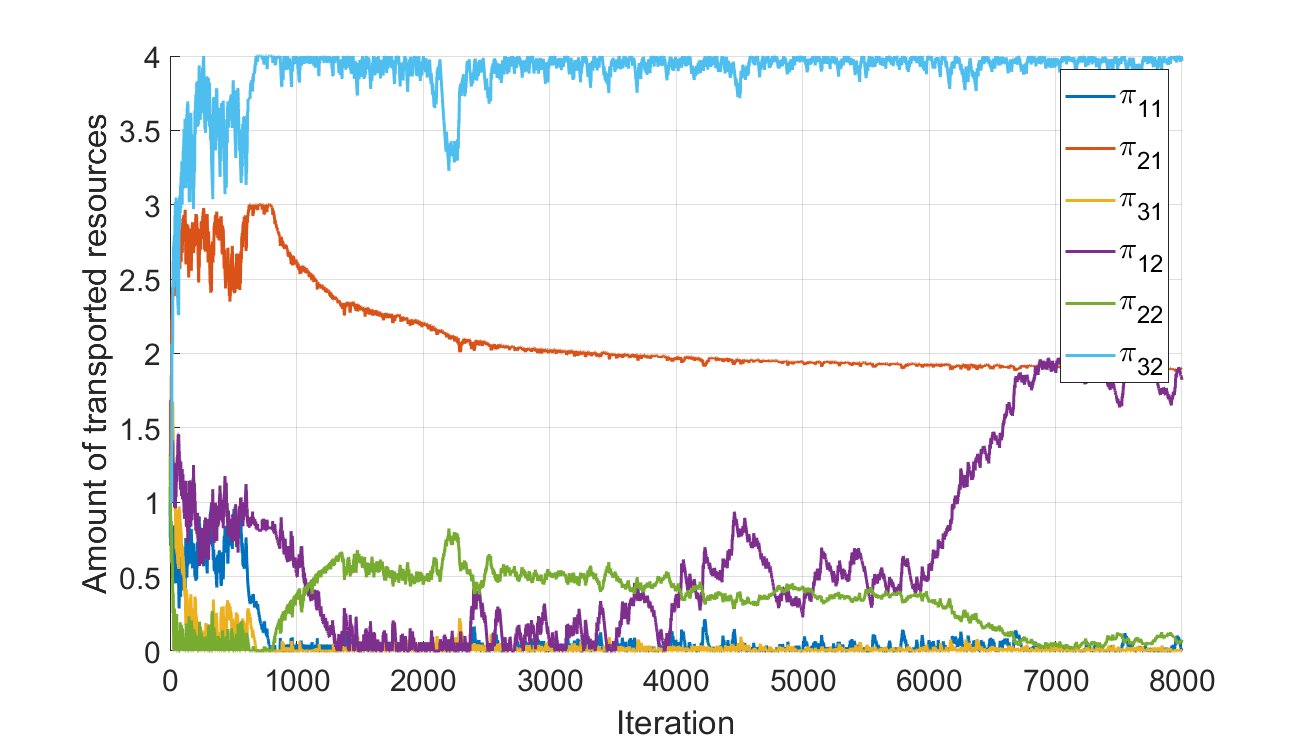}\label{fig:resources1}}
    \subfigure[Transport Utility]{
    \includegraphics[width=0.45\columnwidth]{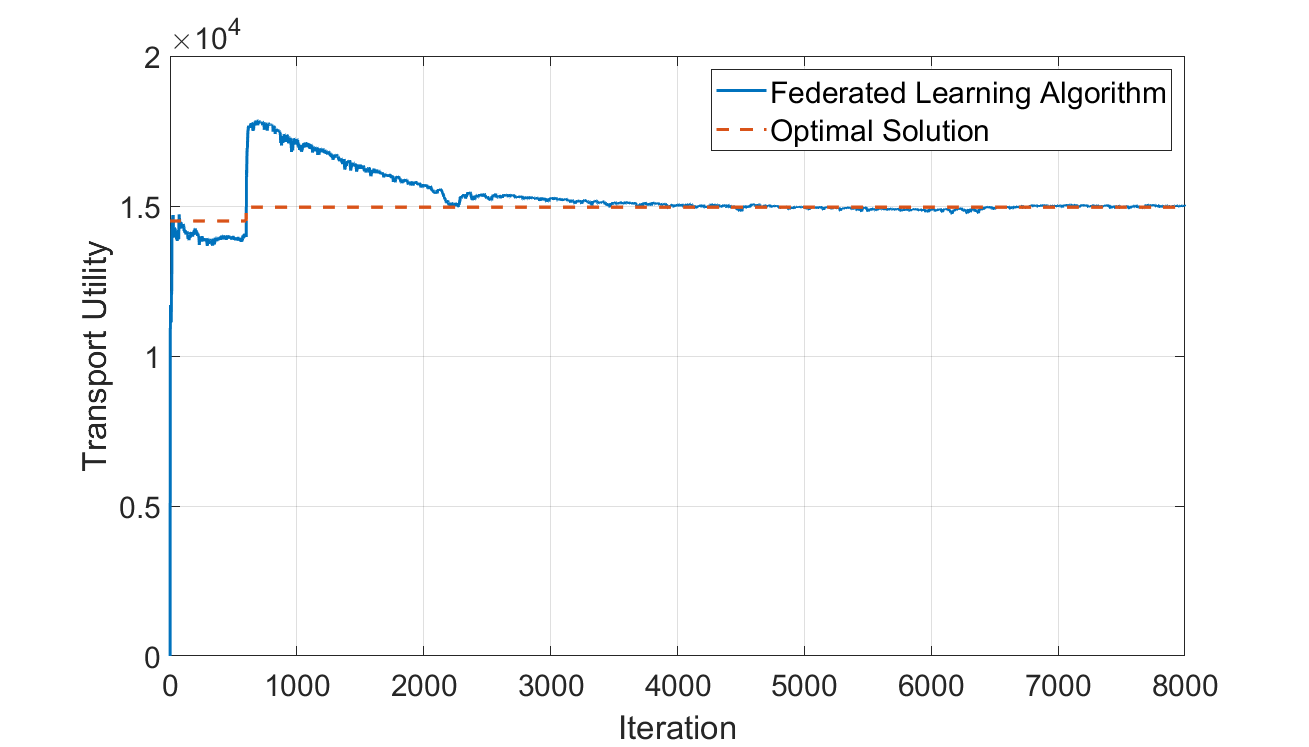}\label{fig:utility1}}
    \subfigure[Received Resources at Targets]{
    \includegraphics[width=0.45\columnwidth]{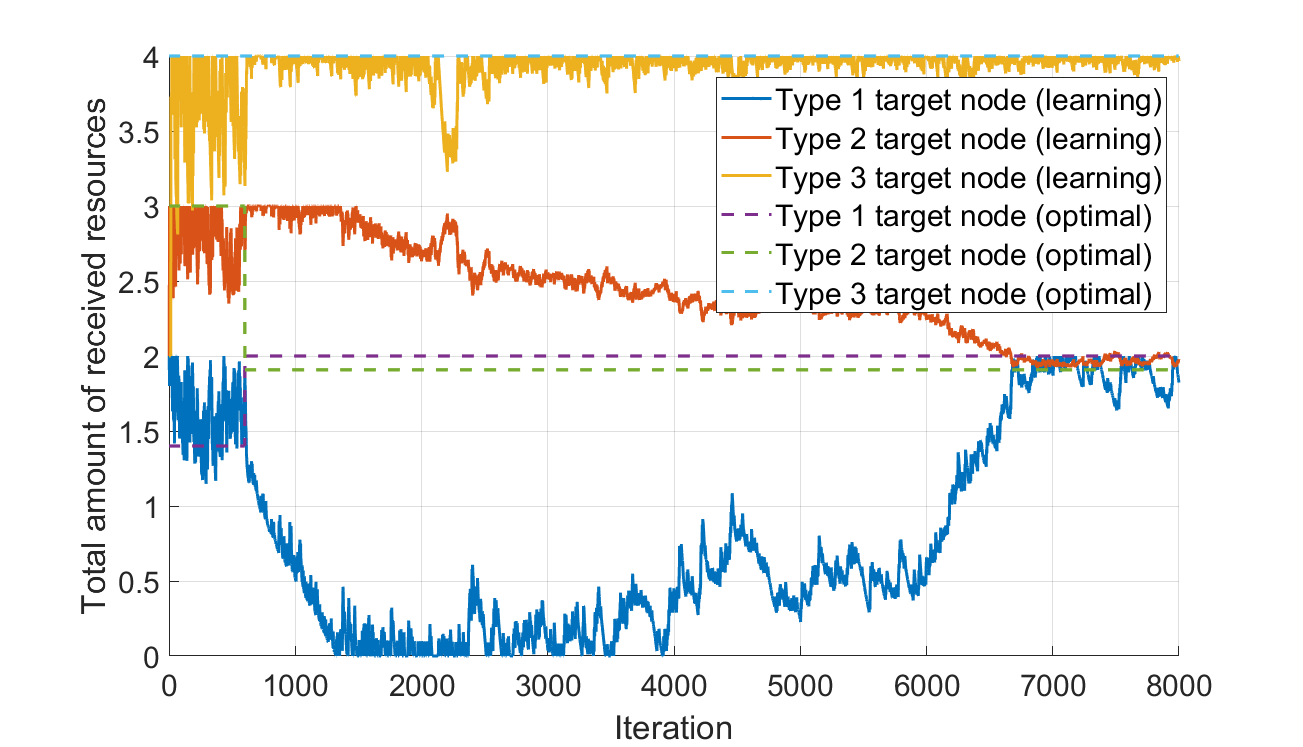}\label{fig:recresources1}}\vspace{-1mm}
  \caption[]{(a): Distribution of the sampled targets' type data incorporating the new distribution; (b): Transport plan before and after the update; (c): Transport utility given by the federated learning algorithm; (d): Resource received by each type of target node before and after the update.}
  \label{fig:ResilienceAlgorithmResults}
\end{figure}

Figure \ref{fig:pdf1} illustrates the updated probability distribution of the target nodes after 8000 iterations. The distribution gets closer to the new distribution but is not completely exact, as it accounts for the samples that are part of the old distribution. Figure \ref{fig:resources1} highlights the adjustments made in the resource allocation scheme after the distribution shift at iteration 600.  This adaptability is further emphasized in Figure \ref{fig:utility1}, where the transport utility indicates a smooth transition and convergence toward a new optimal solution. As shown in Fig. \ref{fig:recresources1}, the federated learning algorithm swiftly adapts to the changes, rebalancing the resource allocation as new data is gathered, ultimately converging to an optimal transport plan in the new environment.

\section{Conclusion}\label{sec:conclusion}

In this paper, we proposed and explored an optimal transport framework designed to handle large-scale, heterogeneous target populations, addressing both scenarios where the target type distribution is known and where it is unknown. Through the development of a distributed ADMM algorithm and a federated learning-based approach, we demonstrated the capability of efficiently computing optimal transport solutions in complex environments. Our federated learning algorithm, in particular, proves to be highly effective in scenarios involving incomplete information, adapting dynamically while preserving the privacy of individual nodes. Future work will develop distributed and privacy-preserving federated learning for large-scale optimal transport.

\bibliographystyle{IEEEtran}
\bibliography{IEEEabrv,references}

\appendix 

\subsection{Proof of Lemma \ref{lemma:lemma1}}\label{app:lemma_1}
\begin{proof}
The optimality condition of problem $\min_{z \in \mathbb{R}^n}f(z;x_{i}) + \frac{1}{2\mu} \| \Pi - z \|^2$ is given by: 
$$
\frac{1}{\mu}(\Pi - z_{\mu}(\Pi;x_{i})) \in \delta f (z_{\mu}(\Pi;x_{i});x_{i}).
$$
The above inclusion implies $g_{f}(z_{\mu}(\Pi;x_{i});x_{i}) \in \delta f(z_{\mu}(\Pi;x_{i});x_{i}),$ such that:
\begin{align*}
&\frac{1}{\mu} \|z_{\mu}(\Pi;x_{i}) - \Pi \|^2 \\ = & \langle g_{f}(z_{\mu}(\Pi;x_{i});x_{i}), \Pi - z_{\mu}(\Pi;x_{i})\rangle\\
=&\langle g_f(\Pi, x_{i}), \Pi- z_{\mu}(\Pi,x_{i}) \rangle \\ &+ \langle g_f(z_{\mu}(\Pi;x_{i});x_{i}) - g_{f}(\Pi;x_{i}), \Pi - z_{\mu}(\Pi;x_{i}) \rangle
\\ \leq & \langle g_{f}(\Pi; x_{i}), \Pi - z_{\mu}(\Pi; x_{i}) \rangle,
\end{align*}
where the inequality follows from the convexity of $f$.
Finally, the lemma follow by applying the Cauchy-Schwarz inequality on the right-hand.
\end{proof}

\subsection{Proof of Theorem \ref{thm:1}}\label{app:thm_1}

For any $\mu>0$, the function 
$
f(\Pi;x_{i}) + \frac{1}{2\mu}\|\Pi-\Pi_0\|^2 
$ 
is strongly convex, we have: 
\begin{align*}
&f(\Pi;x_{i}) + \frac{1}{2\mu}\|\Pi-\Pi_0)\|^2 
\\
\geq &f(z_{\mu}(\Pi_0;x_{i});x_{i}) 
+ \frac{1}{2\mu}\|z_{\mu}(\Pi_0;x_{i})-\Pi_0\|^2 
+ \frac{1}{2\mu}\|z_{\mu}(\Pi_0;x_{i})-\Pi\|^2 \\
= &f_{\mu}(\Pi_0;x_{i}) + \frac{1}{2\mu}\|z_{\mu}(\Pi_0;x_{i})-\Pi\|^2,\quad \forall \Pi  \in \mathbb{R}^n
\end{align*}
Let $\Pi_0 = \Pi_{x}(k), x_{i} = x(k), \Pi  = \Pi^*$, then, the nonexpansive property of the projection operator implies:
\begin{align}
\| \Pi_0 - \Gamma_{\mathcal{L}^{k}}(\Pi_{0}) \|^2 \leq \| \Pi_{0} - \Pi\|^2 - \|\Pi - \Gamma_{\mathcal{L}^{k}}(\Pi_{0})\|^2, \quad  \nonumber\\
\forall \Pi \in \mathcal{L}^{k}, \Pi_{0} \in \mathbb{R}^{n}. \label{THEOREM_eqn1}
\end{align}
recall that $\Gamma$ represents the projection onto the feasible set. Plug it back to the previous inequality evaluated at $x(k)$ and $\mu_k$, we have
\begin{align}
&f(\Pi^{*}, x(k)) + \frac{1}{2\mu_k} \| \Pi(k) - \Pi^{*} \|^{2} \nonumber \\\geq &f_{\mu_k}(\Pi(k); x(k)) + \frac{1}{2\mu_k} \| \Lambda(k) - \Pi^{*} \|^2 \nonumber\\
\geq & f_{\mu_k}(\Pi(k);x(k)) +  \frac{1}{2\mu_k} \| \Gamma_{\mathcal{L}^{k}} (\Lambda(k)) - \Pi^{*} \|^2 + \nonumber\\ &\frac{1}{2\mu_k} \| \Lambda(k) - \Gamma_{\mathcal{L}^{k}}(\Lambda(k))\|^2
\nonumber\\
= &f_{\mu}(\Pi(k);x(k)) + \frac{1}{2\mu_k} \| \Pi(k+1) - \Pi^{*} \|^2 + \nonumber\\ &\frac{1}{2\mu_k} \| \Lambda(k) - \Pi(k+1) \|^2,\label{THEOREM_eqn2}
\end{align}
The inequalities follow from~\eqref{THEOREM_eqn1} with $\Pi_{0} = \Lambda(k)$ and $\Pi = \Pi^{*}$.

Let $\mathbb{I}_{\mu,x_{i}} = \frac{1}{2\mu} \|\Pi_{x} - \Gamma_{\mathcal{L}^{k}}(\Pi_{x}) \|^2$ and~\eqref{THEOREM_eqn2} leads to,
\begin{align*}
&\mu_k(f(\Pi(k);x(k)) - f(\Pi^{*}; x(k)) + \mathbb{I}_{1,x(k)}(\Lambda(k)) - \frac{\mu_k^2}{2}L^{2}_{f,x(k)}  \\
\stackrel{(1)}{\leq} & \mu_k (f(\Pi(k); x(k)) - f(\Pi^{*}; x(k)) + \mathbb{I}_{1,x(k)}(\Lambda(k)) \\ &- \frac{\mu_k^2}{2}\| \nabla f (\Pi(k);x(k)) \|^2
\\\stackrel{(2)}{=} &\mu_k(f(\Pi(k);x(k)) - f(\Pi^{*};x(k)) + \mathbb{I}_{1,x(k)}(\Lambda(k)) \\& + \min_{\Pi_{x_{i}} \in \mathbb{R}^{n}} [\mu_k \langle \nabla f(\Pi(k);x(k)), \Pi_{x_{i}} - \Pi(k) \rangle + \frac{1}{2} \| \Pi_{x_{i}} - \Pi(k) \|^2]
\\ \stackrel{(3)}{\leq}& \mu_k(f(\Pi(k);x(k)) - f(\Pi^{*};x(k)) + \mathbb{I}_{1,x(k)}(\Lambda(k)) \\ &+ \mu_k \langle \nabla f (\Pi(k);x(k)),\Lambda(k) - \Pi(k) \rangle + \frac{1}{2} \| \Lambda(k) - \Pi(k) \|^2
\\ \stackrel{(4)}{=} & \mu_k (f(\Pi(k);x(k)) + \langle \nabla f(\Pi(k); x(k)), \Lambda(k) - \Pi(k) \rangle \\    &+ \frac{1}{2\mu_k} \| \Lambda(k) - \Pi(k) \|^2 - f(\Pi^{*}; x(k)) + \mathbb{I}_{1,x(k)}(\Lambda(k))
\\ \stackrel{(5)}{\leq} & \mu_k (f_{\mu}(\Pi(k);x(k)) - f(\Pi^{*};x(k)) + \mathbb{I}_{1,x(k)} (\Lambda(k))
\\ \stackrel{(6)}{\leq} &\frac{1}{2} \| \Pi(k) - \Pi^{*} \|^2 - \frac{1}{2} \|\Pi(k+1) - \Pi^{*} \|^2,
\end{align*}
where (1) is due to the Lipschitz assumption, (2) is due to the property of the quadratic function, (3) is obtained naturally after the minimizing operator is removed, (4) is a reorganization, (5) is the result of $f_{\mu_k}(\Pi(k);x(k))\ge f(\Pi(k);x(k)) + \langle \nabla f(\Pi(k); x(k)), \Lambda(k) - \Pi(k) \rangle +\frac{1}{2\mu_k} \| \Lambda(k) - \Pi(k) \|^2$, a combination of convexity and optimality, and (6) follows~\eqref{THEOREM_eqn2}.

Taking conditional expectation in $x(k)$ w.r.t. the history $F_{k-1} = \{x(0), ..., x(k-1)\} $ in the last inequality we get:
\begin{align*}
&\mu_k(F(\Pi(k)) - F(\Pi^{*}) + \mathbb{E} [\mathbb{I}_{1,x(k)} (\Lambda(k)) | F_{k-1}] - \frac{\mu_k^2}{2} (L_s+L_t) \\\leq & \frac{1}{2} \mathbb{E} [\| \Pi(k) - \Pi^{*} \|^{2} - \frac{1}{2}\mathbb{E}[\|\Pi(k+1) - \Pi^{*}\|^{2}| F_{k-1}]
\end{align*}

Taking further the expectation over $F_{k-1}$ and summing over $i = 0, ..., k-1$, we get:
\begin{align*}
&\frac{\| \Pi^{0}-\Pi^{*} \|^{2}}{2\sum^{k-1}_{i=0}\mu_i}\\ \geq &\frac{1}{k{\hat \mu}_{1,k}} \sum^{k-1}_{i=0} \mathbb{E} [\mu_i(F(\Pi(i)) - F(\Pi_{x}^{*}))] + \mathbb{E} [\mathbb{I}_{1,x_{i}}(\Lambda(i))] - \frac{\mu_i^{2}}{2} (L_s+L_t)
\\=& \frac{1}{k{\hat \mu}_{1,k}} \sum^{k-1}_{i=0} \mathbb{E}[\mu_i(F(\Pi(i)) - F(\Pi_{x}^{*}))] + \mu_i\mathbb{E}[\mathbb{I}_{\mu,x_{i}}(\Lambda(i))] - \frac{\mu_i^2}{2} (L_s+L_t)
\end{align*}
\begin{equation}\label{THEOREM_eqn3}
\geq \mathbb{E}[F(\hat\Pi(k)) - F(\Pi^{*}))] + \mathbb{E}[\mathbb{I}_{\mu,x_{i}}(\hat\Lambda(k))] - \frac{\hat{\mu}_{2,k}(L_s+L_t)}{2\hat{\mu}_{1,k}}.
\end{equation} 

Inequality~\eqref{THEOREM_eqn3} implies the following upper bound on the optimality gap:
\begin{equation}\label{THEOREM_eqn4}
\mathbb{E}[F(\hat\Pi(k)) - F(\Pi^{*})] \leq \frac{\| \Pi^{0} - \Pi^{*} \|^2 + \frac{\hat{\mu}_{2,k}(L_s+L_t)}{\hat{\mu}_{1,k}}}{2}.
\end{equation}
Let $\nabla F(\Pi^{*}) = \mathbb{E}[\nabla f(\Pi^{*};x_{i})]$, we use the following fact:
\begin{align}
&\mathbb{E}[F(\hat\Pi(k))] - F(\Pi^{*}) \nonumber\\\geq &\mathbb{E}[ \langle F (\Pi^{*}), \hat\Pi(k) - \Pi^{*} \rangle]
\nonumber\\
\geq & - (L_s+L_t) \mathbb{E}[dist(\hat\Pi(k),\mathcal{X})], \quad \forall k \geq 0
\label{eqn:b}
\end{align}
which is derived from the optimality conditions $\langle F (\Pi^{*}), z - \Pi^{*} \geq 0$ for all $z \in \mathcal{X}$, Cauchy-Schwarz and Jensen's inequalities. Denoting $r_{0} = \| \Pi^{0} - \Pi^{*} \|$, combing the last two inequalities with Lemma~\ref{lemma:lemma2}, we obtain:

\begin{align*}
&\mathbb{E}[dist^{2}(\hat\Pi(k),\mathcal{X})] - \xi (L_s+L_t)\hat{\mu}_{1,k} \sqrt{\mathbb{E}[dist^{2}(\hat\Pi(k),\mathcal{X})]}
\\\leq & 2\hat{\mu}_{1,k} \xi \mathbb{E} [F(\hat\Pi(k)) - F(\Pi^{*})] + 2 \hat{\mu}_{1,k} \xi (L_s+L_t)\mathbb{E} [\mathbb{I}_{\mu,x_{i}}(\hat \Lambda(k))]\\
\leq & \frac{\xi r_{0}^{2}}{k} + \xi (L_s+L_t)\hat{\mu}_{2,k}.
\end{align*}
This relation clearly implies an upper bound on the feasibility residual:

\begin{align*}
\sqrt{\mathbb{E}[dist^{2}(\hat\Pi(k),\mathcal{X})]} \leq \xi (L_s+L_t) (3 \hat{\mu}_{1,k}) + \sqrt{\frac{\xi r^{2}_{0}}{k} + \xi (L_s+L_t)^2 \hat{\mu}_{2,k}}.
\end{align*}

Combining this inequality with \eqref{eqn:b}, we obtain the lower bound on the suboptimality gap:
\begin{align*}
&\mathbb{E} [F(\hat \Pi (k))] - F^{*} \\\geq & \xi (L_s+L_t)^2(3\hat{\mu}_{1,k}) - (L_s+L_t) \sqrt{\frac{\xi r_{0}^{2}}{k} + \xi (L_s+L_t)^2 \hat{\mu}_{2,k}}.
\end{align*}
Combined, we deduce our convergence rate results. 

\end{document}